\documentclass[11pt]{article}

\usepackage[final]{acl}

\usepackage{times}
\usepackage{latexsym}

\usepackage[T1]{fontenc}

\usepackage[utf8]{inputenc}

\usepackage{microtype}

\usepackage{inconsolata}

\usepackage{booktabs}       
\usepackage{amsfonts}       
\usepackage{amsmath}
\usepackage{graphicx} 
\usepackage{nicefrac}       
\usepackage{xcolor}         
\usepackage{pifont}
\usepackage{algorithm}
\usepackage{algorithmic}
\usepackage{colortbl}
\usepackage{multirow}
\usepackage{rotating}
\usepackage{subcaption}
\usepackage{bm}
\usepackage{amsthm}

\usepackage{hyperref}
\usepackage{url}
\usepackage{tikz}
\usetikzlibrary{positioning, arrows.meta, shapes.geometric, fit, calc, backgrounds}

\newtheorem{theorem}{Theorem}[section]

\usepackage[most]{tcolorbox}
\tcbuselibrary{listings}
\definecolor{lightgray}{gray}{0.95}
\newtcolorbox{promptbox}[1][]{
  colback=lightgray, 
  colframe=gray,     
  boxrule=0.5pt,     
  arc=1mm,           
  fontupper=\scriptsize\ttfamily, 
  title={Prompt},    
  #1
}

\definecolor{LightRed}{rgb}{0.99,0.91,0.91}
\definecolor{LightBlue}{rgb}{0.88,0.92,0.95}
\definecolor{LightYellow}{rgb}{0.98,0.96,0.89}

\title{Mitigating Query Shortcuts in Chain-of-Thought\\with Context-Aware Decoding}

\author{Shin'ya Yamaguchi \\
  NTT\\
  \texttt{shinya.yamaguchi@ntt.com} \\
  }

\begin{document}
\maketitle
\begin{abstract}
While Chain-of-Thought (CoT) prompting improves the reasoning performance of large language models (LLMs), recent studies have shown that final answers do not always depend on the generated rationales. In particular, LLMs may rely on query-specific shortcuts rather than the information provided by their reasoning steps. In this paper, we investigate whether context-aware decoding (CAD), an existing contrastive decoding method, can mitigate such shortcut behavior when applied to CoT reasoning. We study this application as Rationale Contrastive Decoding (RCD), which contrasts the next-token distributions conditioned on the query with and without its rationale, amplifying tokens whose likelihood increases due to the rationale. We provide a theoretical interpretation of this decoding rule via a KL-regularized reward-maximization objective. Experiments across five reasoning benchmarks demonstrate that RCD improves answer accuracy over standard CoT. Furthermore, intervention analyses using ground-truth and random rationales reveal that RCD increases the behavioral dependence of final answers on the rationale content. Our findings suggest that contrastive decoding offers a simple, training-free approach to mitigating query shortcuts.
\end{abstract}

\section{Introduction}

Chain-of-Thought (CoT) prompting~\cite{Wei_NeurIPS22_CoT, Kojima_NeurIPS22_lets_think_step_by_step} has emerged as an effective inference strategy for improving the reasoning performance of Large Language Models (LLMs).
By generating multiple intermediate reasoning steps, or rationales, before producing a final answer, CoT enables LLMs to solve complex problems.
Beyond improving task performance, these rationales are often interpreted as explanations of the model's predictions.
However, recent studies have revealed that generated rationales do not necessarily reflect the actual decision-making process~\cite{Turpin_NeurIPS23_language_models_dont_always_say, Arcuschin_ICLRW25_cot_is_not_always_faithful, Lewis_EMNLP25_analysing_cot}. 
In particular, LLMs may rely on shortcut information directly from the input query, even when their rationales suggest a different conclusion. 
Such query-shortcut behavior weakens the behavioral dependency between rationales and final answers, raising concerns about the reliability of CoT.

Prior research has investigated this issue by rationale interventions that measure whether changes to reasoning steps affect final predictions~\cite{Lanham_2023_measuring_faithfulness_cot, Turpin_NeurIPS23_language_models_dont_always_say} or improving rationale reliability with additional supervision or rationale selection~\cite{Paul_EMNLP24_measuring_improving_faithfulness_cot, Sun_2026_fact-e}. 
While these studies provide important insights, a complementary question is whether query-shortcut behavior can be mitigated directly during answer decoding, without modifying model parameters or employing external evaluators.
In standard CoT decoding, the query and generated rationale are jointly provided as context for next-token prediction. 
However, this formulation does not explicitly distinguish the contributions of the query and the rationale, allowing the model to generate answers primarily from query-conditioned predictions.

In this work, we investigate the application of Context-Aware Decoding (CAD)~\cite{Shi_NAACL24_cad} to mitigate query-shortcut behavior in CoT. 
CAD is an established contrastive decoding approach originally developed to prioritize contextual evidence over a model's parametric knowledge. 
We examine whether the same principle can encourage LLMs to rely more heavily on self-generated rationales during decoding. 
Specifically, we study Rationale Contrastive Decoding (RCD), an application of CAD in which the context part is a CoT rationale rather than an external document. 
RCD contrasts the next-token distributions conditioned on both the query and rationale with those conditioned only on the query. 
By subtracting query-only logits, this decoding rule increases the relative preference for tokens supported by the rationale, potentially suppressing query-shortcut behavior without additional training. 
We further interpret this established decoding formulation as a KL-regularized reward-maximization objective, in which the reward is the log-density ratio between the rationale-conditioned and query-only distributions.

We empirically investigate RCD through comprehensive experiments across multiple reasoning benchmarks. 
First, intervention analyses with ground-truth and randomly substituted rationales demonstrate that RCD makes final predictions more sensitive to rationale content, providing behavioral evidence of increased rationale–answer dependency. 
Second, RCD consistently improves answer accuracy over standard CoT across five reasoning benchmarks and multiple model families. 
Third, an analysis of rationale correctness reveals that the performance gains primarily stem from correcting answers that fail to follow valid rationales, rather than repairing incorrect reasoning. 
These findings suggest that applying contrastive decoding to CoT can mitigate query-shortcut behavior and improve the utilization of generated rationales. 

\section{Preliminary}\label{sec:preliminary}
\subsection{Next-token Prediction in LLM}
Consider an auto-regressive LLM, parameterized by $\theta$.
Given a query $q$, an LLM generates an output token sequence $\bm{y}=(y_1,\dots,y_L)\in\mathcal{V}^L$ following:
\begin{align}
     p(\bm{y}|q) = \prod^{L}_{i=1}p_\theta(y_i|\bm{y}_{<i},q),\label{eq:llm_cond_dist}
\end{align}
where $L$ is token length, $\mathcal{V}$ is token vocabulary, and $\bm{y}_{<i}$ are preceding output tokens.
$p_\theta(y_i|\bm{y}_{<i},q)$ over $\mathcal{V}$ is computed by the softmax of the model's output $\operatorname{logits}_\theta(y_i|\bm{y}_{<i},q)$ as follows.
\begin{gather}
     p_\theta(y_i|\bm{y}_{<i},q) = \sigma(\operatorname{logits}_\theta(y_i|\bm{y}_{<i},q)) \\
     = \frac{\exp(\operatorname{logits}_\theta(y_i|\bm{y}_{<i},q))}{\sum_{{w}\in\mathcal{V}}\exp(\operatorname{logits}_\theta(y_i=w|\bm{y}_{<i},q))}.
\end{gather}
Each token $y_i$ is generated from $p_\theta(y_i|\bm{y}_{<i},q)$ via a decoding strategy such as greedy decoding and sampling.
We define a function for an arbitrary decoding strategy:
\begin{equation}
     \bm{y} = \operatorname{generate}_\theta(q).\label{eq:llm_generate}
\end{equation}

\subsection{Chain-of-Thought (CoT)}
CoT aims to enhance LLM reasoning by incorporating multi-step rationale generation during decoding~\cite{Wei_NeurIPS22_CoT, Kojima_NeurIPS22_lets_think_step_by_step}.
First, LLMs generate a rationale $r$ from query $q$ with instruction prompt, e.g., \texttt{Let's think step-by-step}~\cite{Kojima_NeurIPS22_lets_think_step_by_step}, as follows.
\begin{equation}
     r = \operatorname{generate}_\theta(q).\label{eq:rationale_generation}
\end{equation}
Adding $r$ to the input context, LLMs generate the final answer by the next-token prediction, similar to Eq.~(\ref{eq:llm_generate}):
\begin{align}
     \bm{y} = \operatorname{generate}_\theta(q,r).\label{eq:llm_cot_generate}
\end{align}
For simplicity, we assume that the rationale and answer generation steps are separated, and that the generated rationale can be intervened by alternative rationales, e.g., ground truth.

In Eq.~\eqref{eq:llm_cot_generate}, each answer token $y_i$ is decoded by the next-token distribution $p_\theta(y_i|\bm{y}_{<i},q,r)$, and thus, it is not guaranteed that the output tokens are based on the rationale $r$, i.e., LLMs can directly answer from the query $q$.
In fact, previous works have reported the unfaithful behavior of CoT, such as post-hoc rationalizing by biasing toward the order of multiple-choice options for the answer~\cite{Turpin_NeurIPS23_language_models_dont_always_say}, silent error correction, where the model implicitly fixes the computational or logical errors in the rationales due to shortcuts~\cite{Arcuschin_ICLRW25_cot_is_not_always_faithful}.
These findings indicate that standard CoT can exhibit weak or inconsistent behavioral dependence on its rationale.

\section{Method}\label{sec:method}
\subsection{Rationale Contrastive Decoding (RCD)}
We investigate whether context-aware decoding can mitigate query-shortcut behavior in CoT by increasing the influence of rationales during answer generation. 
To this end, we study Rationale Contrastive Decoding (RCD), which applies Context-Aware Decoding (CAD;~\citet{Shi_NAACL24_cad}) to self-generated CoT rationales rather than external contexts (Fig.~\ref{fig:method}).
Following CAD, RCD instead reweights this distribution using the density ratio: 
\begin{equation}
     \hat{p}_\theta(y_i) \propto p_\theta(y_i|\bm{y}_{<i},q,r) \times \left(\frac{p_\theta(y_i|\bm{y}_{<i},q,r)}{p_\theta(y_i|\bm{y}_{<i},q)}\right)^\alpha,\label{eq:rcd}
\end{equation}
where $\alpha$ controls the contrastive strength.
During decoding, this is equivalently implemented by the logit-level computation in a batch-parallel manner: 
\begin{align}
     \hat{\operatorname{logits}}_\theta(y_i) :=& (1+\alpha)\operatorname{logits}_\theta(y_i|\bm{y}_{<i},q,r) \nonumber\\
     &- \alpha\operatorname{logits}_\theta(y_i|\bm{y}_{<i},q).\label{eq:rcd_logits}
\end{align}
For $\alpha>0$, the density ratio increases the relative preference for tokens whose probability rises when conditioning on the rationale.
Unlike logit addition, which reinforces tokens favored by both distributions, subtraction emphasizes the information contributed by the rationale beyond the query alone as shown in Section~\ref{sec:experiments}.
This aims to mitigate query-shortcut behavior rather than simply increase agreement with query-only predictions.

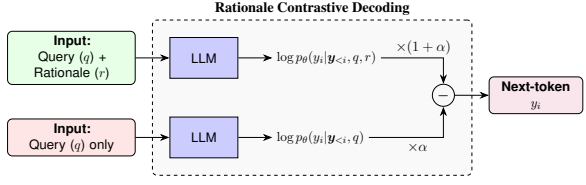
\begin{figure}[t]
    \centering
    \resizebox{\columnwidth}{!}{
    \begin{tikzpicture}[
        font=\sffamily,
        node distance=1cm and 1cm,
        box/.style={draw, rounded corners, fill=blue!10, text centered, minimum height=1cm, minimum width=2.5cm, text width=2.5cm},
        llm/.style={draw, thick, fill=blue!20, text centered, minimum height=1.2cm, minimum width=2cm},
        mathnode/.style={text centered, minimum height=1cm},
        arrow/.style={-{Stealth[scale=1.2]}, thick},
        dashedarrow/.style={-{Stealth[scale=1.2]}, thick, dashed},
        coefficient/.style={font=\large, text centered},
    ]

    \node[box, fill=green!10, text width=3.5cm] (input_tgt) {\textbf{Input:}\\Query ($q$) + Rationale ($r$)};
    \node[llm, right=of input_tgt] (llm_tgt) {LLM};
    \node[mathnode, right=of llm_tgt] (prob_tgt) {$\log p_\theta(y_i|\bm{y}_{<i},q,r)$};

    \node[box, fill=red!10, text width=3.5cm, below=1cm of input_tgt] (input_ctx) {\textbf{Input:}\\Query ($q$) only};
    \node[llm, right=of input_ctx] (llm_ctx) {LLM};
    \node[mathnode, right=of llm_ctx] (prob_ctx) {$\log p_\theta(y_i|\bm{y}_{<i},q)$};

    \node[circle, draw, thick, right=1.5cm of prob_tgt, yshift=-1.1cm, text centered, inner sep=2pt] (minus) {$\boldsymbol{-}$};
    \node[box, fill=purple!10, right=1cm of minus, minimum width=2cm] (output) {\textbf{Next-token}\\$y_i$};

    \draw[arrow] (input_tgt) -- (llm_tgt);
    \draw[arrow] (llm_tgt) -- (prob_tgt);
    \draw[arrow] (prob_tgt.east) -- ++(0.5,0) -| (minus.north);

    \draw[arrow] (input_ctx) -- (llm_ctx);
    \draw[arrow] (llm_ctx) -- (prob_ctx);
    \draw[arrow] (prob_ctx.east) -- ++(0.5,0) -| (minus.south);

    \draw[arrow] (minus) -- node[above] {} (output);
    
    \node (mult_top) [coefficient, right=0.3cm of prob_tgt.east, yshift=0.3cm] {$\times(1 + \alpha)$};
    \node (mult_bot) [coefficient, right=1cm of prob_ctx.east, yshift=-0.3cm] {$\times\alpha$};
    
    \begin{scope}[on background layer]
        \node[draw, dashed, fill=gray!5, rounded corners, inner sep=0.5cm, fit=(llm_tgt) (llm_ctx) (minus)] (bg) {};
        \node[above, font=\large\bfseries] at (bg.north) {Rationale Contrastive Decoding};
    \end{scope}
    
    \end{tikzpicture}
    }
    \caption{
    Overview of Rationale Contrastive Decoding (RCD), which applies context-aware decoding to CoT answer generation. Instead of standard decoding, RCD contrasts the rationale-conditioned pass and the query-only pass to explicitly strengthen the dependency on the rationales, without any parameter updates.}
    \label{fig:method}
    \vspace{-5mm}
\end{figure}

\subsection{Theoretical Insight}
Eq.~(\ref{eq:rcd}) can be interpreted as the closed-form solution to the following KL-regularized reward maximization problem~\cite{Schulman_ICML15_TRPO, Schulman_ICML15_PPO}:
\begin{equation}
    \max_{\pi} \mathbb{E}_{\pi}[R] - \beta \mathbb{D}_\mathrm{KL}[\pi || \pi_\mathrm{ref}],\label{eq:rcd_as_reward_max}
\end{equation}
where $R=\log\frac{p_\theta(y_i|\bm{y}_{<i},q,r)}{p_\theta(y_i|\bm{y}_{<i},q)}$ is reward function, $\pi_\mathrm{ref}=p_\theta(y_i|\bm{y}_{<i},q,r)$ is reference policy, and $\beta$ is a hyperparameter modulating the KL-penalty term.
\begin{theorem}\label{prop:rcd}
Sampling by Eq.~(\ref{eq:rcd}) is equivalent to sampling from the optimal policy $\pi^*$ for Eq.~(\ref{eq:rcd_as_reward_max}).
\end{theorem}

\begin{proof}
See Appendix~\ref{sec:appendix_proof}.
\end{proof}

Intuitively, maximizing Eq.~\eqref{eq:rcd_as_reward_max} means finding the next-token distribution $\pi$ that yields a token clearly highlighting the difference from the query-only distribution while preserving the likelihood of the original query-rationale conditional distribution.
As a result, the LLM is expected to generate an answer by relying heavily on the CoT rationale $r$.

\begin{table}[t]
\centering
\caption{Intervention Analysis with Ground-Truth (GT) and Irrelevant (Random) Rationales.
}
\label{tb:intervention}
\vspace{-2mm}
\resizebox{0.9\columnwidth}{!}{
\begin{tabular}{@{}lccc@{}}
\toprule    
 & \multicolumn{3}{c}{\textbf{GSM8K}} \\
 & Self ($\uparrow$) & GT ($\uparrow$) & Random ($\downarrow$) \\
\midrule
&\multicolumn{3}{@{}c}{\cellcolor{LightRed}\textbf{Qwen3-8B}} \\
Direct & 50.14$\pm$.17 & $-$ & $-$ \\
CoT & 89.16$\pm$.13 & 94.93$\pm$.04 & 43.35$\pm$.03 \\
Logit Addition & 81.26$\pm$.40 & 83.41$\pm$.68 & 50.82$\pm$.72 \\
RCD & \textbf{89.91$\pm$.28} & \textbf{95.82$\pm$.01} & \textbf{27.79$\pm$.03} \\
\bottomrule
\end{tabular}
}
\vspace{-3mm}
\end{table}

\section{Experiments}\label{sec:experiments}
\subsection{Settings}\label{sec:exp_setting}
We used five representative reasoning-intensive benchmark datasets: GSM8K~\cite{Cobbe_2021_GSM8K}, MATH~\cite{Hendrycks_NeurIPS21_MATH}, CommonsenseQA~\cite{Talmor_NAACL19_CommonsenseQA}, AQuA-RAT~\cite{Ling_2017_AQuA}, and MMLU-Pro~\cite{Wang_2024_mmlu-pro}.
For LLMs, we used Qwen3~\cite{Yang_2025_qwen3}, Gemma3~\cite{Google_2025_gemma3}, and Llama-3.1~\cite{meta_2024_llama3}.
To examine the effect of contrastive subtraction, we additionally evaluate Logit Addition, which sums the rationale-conditioned and query-only logits, corresponding to a Product-of-Experts combination.
This baseline allows us to investigate whether subtracting query-only predictions is more effective than reinforcing them.
We set $\alpha=1$ for RCD by grid-searching on the GSM8K train subset.
For sampling, we used the official default parameters (temperature: 0.7, top\_k: 20, top\_p: 0.8) for Qwen3.
CoT and RCD shared the same rationales.
All experiments were run three times on an NVIDIA H100 with 80 GB of VRAM.

\subsection{Proof-of-Concept by Intervention Analysis}\label{sec:exp_poc}
We conduct an intervention analysis on GSM8K by replacing the rationale $r$ in Eq.~\eqref{eq:llm_cot_generate} with either ground-truth (GT) or random rationales.
If a decoding method sensitive to rationales, GT rationales should improve performance and irrelevant random rationales should degrade it.
We use the GT rationales provided in GSM8K and construct random rationales by shuffling them across examples.

Table~\ref{tb:intervention} shows the intervention results. Standard CoT improves over direct answering, but still keeps non-trivial accuracy even with random rationales, suggesting that it can often ignore the given rationale and rely on shortcut biases.
In contrast, RCD achieves higher accuracy with GT rationales and substantially lower accuracy with random rationales, indicating stronger behavioral dependence on rationale content.
Notably, Logit Addition degrades accuracy with both self-generated and GT rationales while retaining higher accuracy with random rationales.
These contrasting results suggest that subtracting query-only logits, rather than reinforcing them, is effective in increasing rationale–answer dependency. 

\begin{table}[t]
\centering
\caption{
Performance on Reasoning Benchmarks.
}
\label{tb:benchmark}
\vspace{-2mm}
\resizebox{0.95\linewidth}{!}{
\begin{tabular}{@{}lccccc}
\toprule
 & \textbf{GSM8K} & \textbf{MATH} & \textbf{CommonsenseQA} & \textbf{AQuA-RAT} & \textbf{MMLU-Pro} \\
\midrule
&\multicolumn{5}{@{}c}{\cellcolor{LightRed}\textbf{Qwen3-8B}} \\
Direct & 50.14$\pm$.17 & 1.73$\pm$.23 & 71.27$\pm$.81 & 58.88$\pm$.74 & 35.29$\pm$.48 \\
CoT & 89.16$\pm$.13 & 17.20$\pm$.35 & 78.29$\pm$.53 & 68.50$\pm$.68 & 50.24$\pm$.52 \\
Logit Addition & 81.26$\pm$.40 & 15.59$\pm$.66 & 77.19$\pm$.56 & 63.38$\pm$.53 & 48.62$\pm$.65 \\
RCD & \textbf{89.91$\pm$.28} & \textbf{27.60$\pm$.51} & \textbf{80.13$\pm$.47} & \textbf{70.41$\pm$.40} & \textbf{52.28$\pm$.41} \\

\bottomrule
\end{tabular}
}
\vspace{-3mm}
\end{table}

\subsection{Multiple Reasoning Benchmarks}\label{sec:exp_benchmark}
Table~\ref{tb:benchmark} shows the performance across the benchmarks.
We see that RCD outperforms the standard CoT and Logit Addition for all cases.
This suggests that RCD's increasing reliance on rationales is effective, particularly for reasoning-intensive tasks.
We also analyze RCD on varying $\alpha$, multiple LLM backbones, and the combination with Self-Consistency in the Appendix.

\begin{figure}[t]
    \centering
    \includegraphics[width=0.9\linewidth]{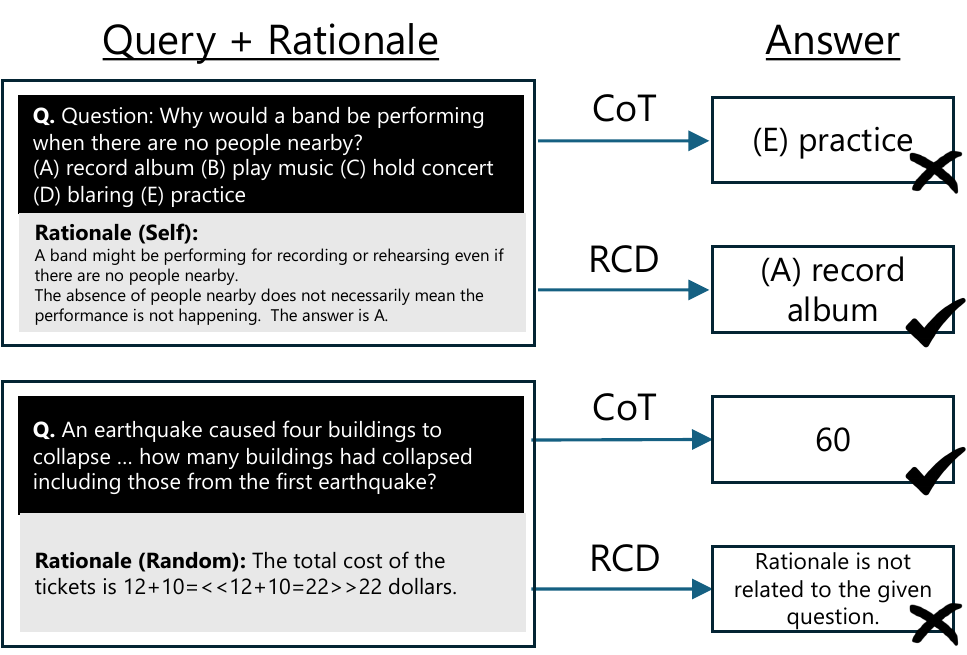}
    \caption{
    \textbf{Qualitative Evaluation with Qwen3-8B.} 
    Top: Given a self-generated rationale supporting answer (A), standard CoT predicts (E), whereas RCD follows the rationale and predicts (A). Bottom: Given an irrelevant rationale randomly sampled from another question, standard CoT produces the correct answer, whereas RCD indicates that the rationale is unrelated to the query. These examples illustrate RCD's stronger sensitivity to rationale content.
    }
    \label{fig:qualitative}
    \vspace{-5mm}
\end{figure}

\subsection{Qualitative Analysis}\label{sec:exp_qual_analysis}
Figure~\ref{fig:qualitative} shows examples of RCD in CoT.
In the top example, CoT ignores the rationale's conclusion, i.e., ``the answer is A'', whereas RCD generates the correct answer based on the conclusion.
The bottom case also highlights the RCD's reliance: CoT generates the correct answer by ignoring the rationale, whereas RCD increased reliance and answered that it cannot draw a conclusion from it.

\begin{table}[t]
\centering
\caption{
Prediction Transitions from CoT to RCD
}
\label{tb:human_eval}
\vspace{-2mm}
\resizebox{0.95\linewidth}{!}{
\begin{tabular}{@{}lcccc}
\toprule
 \textbf{Rationale} & \textbf{CoT Acc.} & \textbf{RCD Acc.} & \textbf{Wrong$\to$Correct} & \textbf{Correct$\to$Wrong} \\
\midrule
&\multicolumn{4}{@{}c}{\cellcolor{LightRed}\textbf{Qwen3-8B}} \\
Correct ($N$=1,201) & 97.25 & 98.08 & 10/33 (30.30\%) & 0/1,168 (0.00\%) \\
Incorrect ($N$=118) & 6.78  & 6.78  & 0/110 (0.00\%) & 0/8 (0.00\%) \\
\bottomrule
\end{tabular}
}
\vspace{-3mm}
\end{table}

\subsection{Analysis with Human Annotation}\label{sec:exp_human_analysis}
To examine how rationale correctness affects RCD, we annotated the rationales generated by Qwen3-8B on the GSM8K test set (1,319 samples) and analyzed prediction transitions from CoT to RCD. Table~\ref{tb:human_eval} summarizes the results. 
Among 1,201 correct rationales, RCD corrects 10 of 33 CoT errors (30.30\%) without corrupting any correct predictions. 
In contrast, among 118 incorrect rationales, RCD neither corrects any of the 110 CoT errors nor changes any of the eight correct predictions.
These results indicate that RCD's accuracy gains primarily stem from improving adherence to correct rationales, while increased reliance on incorrect rationales can occasionally be detrimental.

\section{Related Work}\label{sec:relatedwork}

Recent studies have examined CoT faithfulness through rationale interventions~\cite{Lanham_2023_measuring_faithfulness_cot} and methods for measuring or improving rationale-answer alignment~\cite{Paul_EMNLP24_measuring_improving_faithfulness_cot, Sun_2026_fact-e}.
These approaches primarily focus on faithfulness evaluation or require additional mechanisms (e.g., preference tuning and multiple-rationale evaluations), rather than directly providing comparable answer-decoding baselines.
In contrast, we investigate the application of existing Context-Aware Decoding~\cite{Shi_NAACL24_cad} to CoT rationales, without training or external evaluators.
We therefore focus on decoding baselines and separately examine compatibility with Self-Consistency~\cite{Wang_ICLR23_sc} in Appendix~\ref{sec:append_self_consistency}.

\section{Conclusion}\label{sec:conclusion}
In this work, we investigated the application of context-aware contrastive decoding to mitigate query-shortcut behavior in Chain-of-Thought reasoning. 
By contrasting rationale-conditioned and query-only next-token distributions, Rationale Contrastive Decoding (RCD) encourages LLMs to generate answers that depend more strongly on their rationales. 
Our experiments demonstrate increased behavioral rationale–answer dependency and consistent accuracy improvements over standard CoT across multiple reasoning benchmarks. 
These findings suggest that existing contrastive decoding techniques can improve the utilization of CoT rationales without additional training or external evaluators. 
Nevertheless, stronger rationale dependency does not guarantee faithful reasoning, particularly when the rationales themselves are incorrect. 
Our study thus highlights both the potential and the limitations of decoding-level interventions for mitigating query shortcuts in CoT.

\clearpage

\section*{Acknowledgment}
The initial ideas for this project were developed in February 2026 in collaboration with Manami Hasegawa of Japan Women's University, who was working as an NTT intern at the time.

\section*{Limitations}
First, although our rationale interventions show that RCD increases the behavioral dependency of final answers on the provided rationales, they do not establish that the rationales faithfully reflect the model's internal reasoning process. 
In particular, changes in accuracy under ground-truth or randomly shuffled rationales measure sensitivity to rationale content but do not necessarily indicate more correct or faithful reasoning. 
Therefore, our findings should be interpreted as evidence of stronger rationale–answer dependency rather than a general improvement in reasoning reliability.

Second, stronger rationale dependency can be detrimental when rationales are incorrect or incomplete. 
Since RCD does not verify rationale correctness, it may amplify erroneous reasoning while suppressing useful information from the query-only distribution. 
Future work could explore adaptive contrastive weighting based on rationale quality.

We further discuss potential societal and deployment risks in Appendix~\ref{sec:appendix_potential}.

\bibliography{ref}

\clearpage
\appendix
\section{Proof of Theorem~\ref{prop:rcd}}\label{sec:appendix_proof}
\begin{proof}
Consider the general KL-regularized reward maximization objective~\cite{Schulman_ICML15_TRPO,Schulman_ICML15_PPO,Rafailov_NeurIPS23_DPO}:
\begin{equation}
    \max_{\pi} \mathbb{E}_{\pi}[R(s,a)] - \beta \mathbb{D}_\mathrm{KL}[\pi(a|s) || \pi_\mathrm{ref}(a|s)],\label{eq:klmax_problem}
\end{equation}
where $R$ is a reward function, $s$ is a state (input context), and $a$ is an action (output).
Obviously, Eq.~(\ref{eq:rcd_as_reward_max}) is a special case of Eq.~(\ref{eq:klmax_problem}).
From Appendix A.1 of \citet{Rafailov_NeurIPS23_DPO}, the optimal policy $\pi^*(a|s)$ for this objective is given by
\begin{equation}
    \pi^*(a|s) \propto \pi_\mathrm{ref}(a|s) \exp\left(\frac{1}{\beta} R(s,a)\right).\label{eq:klmax_optimal}
\end{equation}
In the maximization of Eq.~(\ref{eq:rcd_as_reward_max}), given $a=y_i$ and $s=(\bm{y}_{i<},q,r)$, we can set each component in Eq.~(\ref{eq:klmax_problem}) as $\pi_\mathrm{ref}(a|s) = p_\theta(y_i|\bm{y}_{<i},q,r)$, $R(s,a)=\log\frac{p_\theta(y_i|\bm{y}_{<i},q,r)}{p_\theta(y_i|\bm{y}_{<i},q)}$, and $\beta = 1/\alpha$.
Substituting them into Eq.~\eqref{eq:klmax_optimal} yields
\begin{align*}
&\pi^*(a|s) \\
&\propto p_\theta(y_i|\bm{y}_{<i},q,r) \exp\left(\alpha \log\frac{p_\theta(y_i|\bm{y}_{<i},q,r)}{p_\theta(y_i|\bm{y}_{<i},q)}\right) \\
&= p_\theta(y_i|\bm{y}_{<i},q,r)\left(\frac{p_\theta(y_i|\bm{y}_{<i},q,r)}{p_\theta(y_i|\bm{y}_{<i},q)}\right)^\alpha \\
&= \text{Eq.~\eqref{eq:rcd}}.
\end{align*}
Therefore, sampling by Eq.~(\ref{eq:rcd}) is indeed equivalent to sampling from the optimal policy $\pi^*(a|s)$ for the maximization in Eq.~\eqref{eq:rcd_as_reward_max}.
\end{proof}
Note that while this theoretical framework is shared with RED~\cite{Yamaguchi_CVPR26_RED}, which multiplies the separated image-conditioned and rationale-conditioned distributions as product-of-experts for multimodal CoT, our RCD and RED have different formalizations and objectives.

\section{Potential Risks}\label{sec:appendix_potential}

First, RCD has potential dual-use risks. By making answers more consistent with explicit rationales, the method could be misused to produce more persuasive but misleading explanations, such as rationalized misinformation, deceptive arguments, or fabricated justifications. While our work does not target such applications, future deployment of rationale-sensitive decoding methods should consider monitoring, provenance tracking, and safeguards against persuasive misuse.

Second, the fairness and robustness implications of RCD remain underexplored. Our experiments focus on standard English reasoning benchmarks and do not evaluate multilingual, culturally grounded, or demographically sensitive scenarios. If the generated rationales encode social biases or culturally narrow assumptions, increasing the model's reliance on them may reinforce biased outputs. Future work should evaluate RCD on more diverse languages, domains, and user populations before deployment in broad-facing applications.

\section{Artifact Licenses and Terms of Use}

We use existing public artifacts only for research evaluation. Specifically, our experiments use public reasoning benchmarks, GSM8K, MATH, CommonsenseQA, AQuA-RAT, and MMLU-Pro, and the Qwen3/Gemma3/Llama3.1 model families. We do not create, modify, or redistribute these artifacts as part of this work; we only use them to evaluate inference-time decoding methods. Table~\ref{tab:artifact_licenses} summarizes the licenses or terms reported by the corresponding public artifact pages.

\begin{table}[t]
\centering
\small
\resizebox{\columnwidth}{!}{
\begin{tabular}{lll}
\toprule
Artifact & Type & License / Terms \\
\midrule
GSM8K & Dataset & MIT License \\
MATH & Dataset & MIT License \\
CommonsenseQA & Dataset & MIT License \\
AQuA-RAT & Dataset & MIT License \\
MMLU-Pro & Dataset & Apache License 2.0 \\
Qwen3 & Model & Apache License 2.0 \\
Gemma3 & Model & Gemma Terms of Use \\
Llama3.1 & Model &  LLAMA 3.1 Community License \\
\bottomrule
\end{tabular}
}
\caption{Licenses or terms of use for the public artifacts used in this work, based on their public dataset or model cards.}
\label{tab:artifact_licenses}
\end{table}

Our use of these artifacts is consistent with their intended research use: the datasets are used as evaluation benchmarks, and the pretrained language models are used as fixed inference backbones. We do not fine-tune the models, release model derivatives, or redistribute benchmark data. We cite the original papers and artifact pages associated with each dataset and model. Users who reproduce or extend our experiments should consult the latest official artifact pages and comply with the corresponding licenses and terms of use.

\section{Use of AI Assistant}
We used ChatGPT to polish the structure and grammar of the manuscript.
We used Claude Code as a supporting tool for building experimental code.

\section{Instruction Prompts}

\begin{promptbox}[title={Instruction for Rationale Generation}]
\textbf{\#\# Instruction}

For the provided question, please provide detailed step-by-step reasoning descriptions as the rationale to output the answer. Please compose your output with short and clear sentences. Output only reasoning descriptions and DO NOT output the final answer.

\textbf{\#\# Question}

[Placeholder for Question]

\textbf{\#\# Response}

\end{promptbox}

\begin{promptbox}[title={Instruction for Answer Generation}]
\textbf{\#\# Instruction}

Answer the following question based on the given rationales. Please directly output answer in short.

\textbf{\#\# Rationales}

[Placeholder for Rationales]

\textbf{\#\# Question}

[Placeholder for Question]

\textbf{\#\# Response}

\end{promptbox}

\begin{promptbox}[title={Instruction for Direct Baseline}]
\textbf{\#\# Instruction}

Answer the following question directly in short without reasoning steps.

\textbf{\#\# Question}

[Placeholder for Question]

\textbf{\#\# Response}

\end{promptbox}

\section{Grid Search for $\alpha$}
For the validation dataset, we randomly extracted 1,000 samples from the training set of GSM8K.
Then, we searched $\alpha$ from the search space of $\{0.1, 0.3, 0.5, 1, 3, 5\}$.

\subsection{Ablation Study}\label{sec:exp_ablation}
\noindent\textbf{Effects of $\alpha$.}
Figure~\ref{fig:ablation_alpha} shows the validation performance when varying $\alpha$.
We observed that increasing $\alpha$ enhances the reliance on the rationales, but accuracy actually decreases once a certain value is exceeded.
This is because if the query-only conditions in Eq.~\eqref{eq:rcd_logits} are significantly reduced with larger $\alpha$, the performance of the query itself also declines.

\noindent\textbf{Model Generality.}
We show the performance when varying the model size and families in Table~\ref{fig:ablation_size}.
RCD stably performed on multiple models.
The performance gain over CoT increases with model size, suggesting that RCD can significantly improve answer quality with stronger rationales when a model has strong reasoning capabilities.

\begin{figure}[t]
    \centering
      \includegraphics[width=\columnwidth]{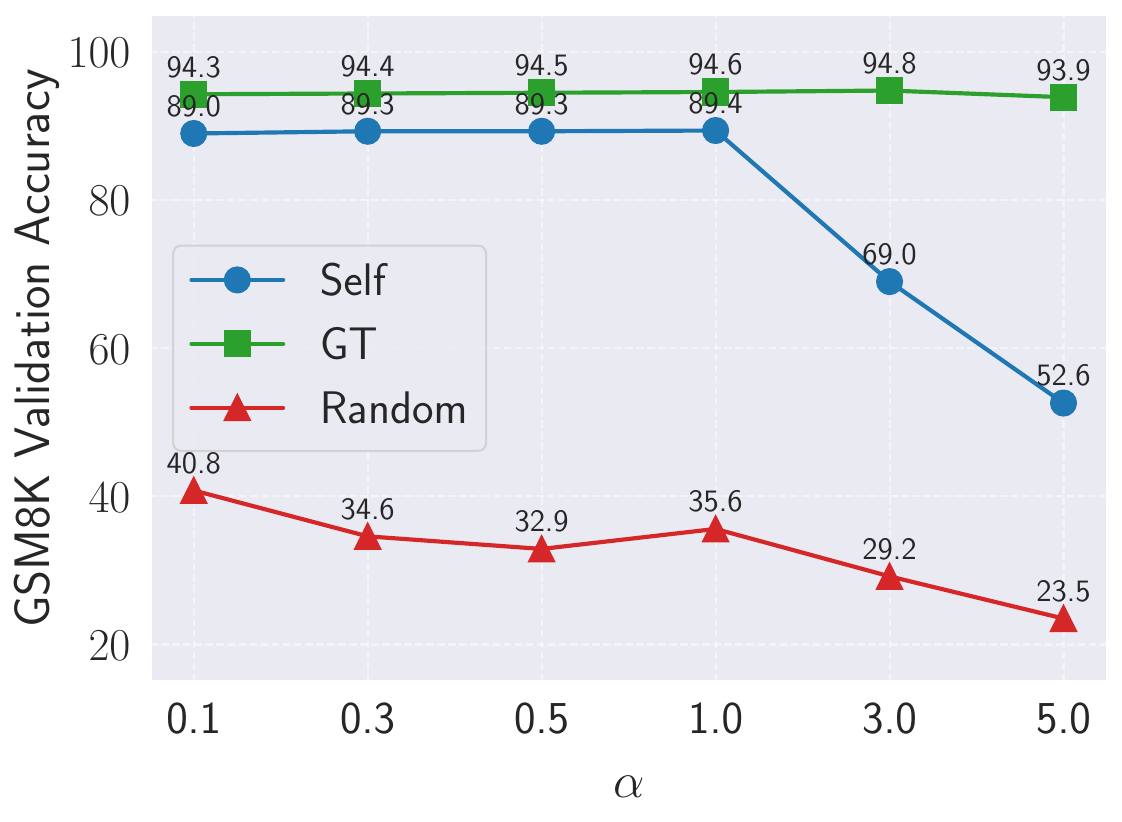}
      \vspace{-8mm}
      \caption{Effects of $\alpha$ on GSM8K validation accuracy (\%) using Qwen3-8B. We evaluate self-generated (Self), ground-truth (GT), and randomly shuffled (Random) rationales on 1,000 validation examples. Increasing ($\alpha$) generally reduces accuracy with random rationales, while excessively large values degrade accuracy with self-generated rationales, illustrating the trade-off between rationale reliance and answer accuracy.}\label{fig:ablation_alpha}
\end{figure}

\begin{table}[t]
    \centering
    \captionof{table}{GSM8K test accuracy (\%) across different model families and sizes, including Qwen3, Gemma3, and Llama-3.1. We compare direct answering (Direct), standard CoT, and RCD with ($\alpha=1$). CoT and RCD share the same generated rationales. RCD consistently improves accuracy over standard CoT across all evaluated models, demonstrating its applicability across model architectures and scales.}\label{fig:ablation_size}
    \resizebox{\linewidth}{!}{
    \begin{tabular}{lccc}
    \toprule    
     & \multicolumn{3}{c}{\textbf{GSM8K}} \\
     Model & Direct & CoT & RCD \\
    \midrule
    Qwen3-0.6B  & 23.49$\pm$.11 & 47.44$\pm$.18 & \textbf{47.96}$\pm$.22 \\
    Qwen3-8B    & 50.14$\pm$.17 & 89.16$\pm$.13 & \textbf{89.91$\pm$.28} \\
    Qwen3-14B   & 38.38$\pm$.43 & 89.61$\pm$.22 & \textbf{90.31$\pm$.20} \\
    Qwen3-32B   & 48.23$\pm$.37 & 91.31$\pm$.33 & \textbf{93.27$\pm$.28} \\
    \midrule
    Gemma3-4B   &  8.59$\pm$.03 & 80.69$\pm$.13 & \textbf{81.93$\pm$.03} \\
    Gemma3-12B  & 18.38$\pm$.07 & 88.41$\pm$.21 & \textbf{90.18$\pm$.16} \\
    Llama3.1-8B & 54.63$\pm$.20 & 84.35$\pm$.12 & \textbf{85.59$\pm$.22} \\
    \bottomrule
    \end{tabular}
    }
    \vspace{-3mm}
\end{table}

\begin{table}[t]
\centering
\caption{
Computational Analysis (GSM8K)
}
\label{tb:computational}
\vspace{-2mm}
\resizebox{0.95\linewidth}{!}{
\begin{tabular}{@{}lcc}
\toprule
 & \textbf{Latency (s/query)} & \textbf{Peak VRAM (GB)} \\
\midrule
&\multicolumn{2}{@{}c}{\cellcolor{LightRed}\textbf{Qwen3-8B}} \\
Direct & 1.506 & 15.33\\
CoT & 2.823 & 15.64 \\
RCD & 2.864 & 15.86 \\
\bottomrule
\end{tabular}
}
\vspace{-3mm}
\end{table}

\section{Computational Analysis}
RCD requires two model evaluations at each answer-decoding step. To quantify the practical cost, we measured end-to-end latency and peak VRAM on Qwen3-8B/GSM8K under the same hardware and decoding settings.
Table~\ref{tb:computational} shows the results. 
Relative to CoT, RCD increases the observed end-to-end latency by only 1.5\% and peak VRAM by 1.4\% in our batch-parallel implementation. Thus, although RCD increases answer-stage model computation, its measured wall-clock and memory overhead are small in this setting.

\begin{table}[htpb]
    \centering
    \caption{
        Comparison with Self-Consistency (SC) on GSM8K using Qwen3-8B.
        $k$ denotes the number of sampled rationales.
        Latency is measured end-to-end per query.
    }
    \label{tab:self_consistency}
    \resizebox{\columnwidth}{!}{
    \begin{tabular}{lccc}
        \toprule
        Method & $k$ & Accuracy (\%) & Latency (s/query) \\
        \midrule
        CoT        & 1 & 89.16 & 2.82 \\
        RCD        & 1 & 89.91 & 2.86 \\
        SC         & 3 & 89.58 & 8.63 \\
        SC + RCD   & 3 & \textbf{90.79} & 8.65 \\
        \bottomrule
    \end{tabular}
    }
\end{table}

\section{Compatibility with Self-Consistency}
\label{sec:append_self_consistency}
We further examine whether RCD can be combined with Self-Consistency (SC; \citet{Wang_ICLR23_sc}), which aggregates predictions from multiple sampled rationales. 
Unlike SC, which improves answer selection through multiple reasoning paths, RCD modifies the answer decoding distribution conditioned on each rationale. We evaluate their combination on GSM8K using Qwen3-8B with three sampled rationales ($k=3$).

Table~\ref{tab:self_consistency} presents the results. 
Standard SC achieves 89.58\% accuracy, whereas combining SC with RCD improves accuracy to 90.79\%.
Similarly, in the single-rationale setting, RCD improves standard CoT from 89.16\% to 89.91\%. 
These results suggest that RCD is complementary to multi-sample aggregation. 
Furthermore, RCD incurs only a 1.5\% increase in end-to-end latency over standard CoT, indicating that its additional computational overhead is small in this setting.

\section{Human Annotation}
We annotated binary labels for the correctness of the generated rationales in the analysis of Section~\ref{sec:exp_human_analysis}.
For each triple of query, generated rationale, and ground truth answer, we asked ``Is the rationales behind this query and answer correct?''
We used one annotator for this annotation process.

\end{document}